\documentclass[11pt,letterpaper]{article}

\usepackage{microtype}
\usepackage{graphicx}
\usepackage{subfigure}
\usepackage{booktabs}

\usepackage{amsfonts}
\usepackage{amsmath}
\usepackage{amssymb}
\usepackage{amsthm}
\usepackage{bbm}
\usepackage{bm}
\usepackage{boxedminipage}
\usepackage{float}
\usepackage{ifthen}
\usepackage{listings}
\usepackage{mathtools}
\usepackage{multirow}
\usepackage{nicefrac}
\usepackage{xparse}
\usepackage{xspace}

\usepackage{textcase}
\usepackage{soul}
\usepackage{indentfirst}

\usepackage{tikz}

\newtheorem*{rep@theorem}{\rep@title}
\newcommand{\newreptheorem}[2]{%
\newenvironment{rep#1}[1]{%
 \def\rep@title{\theoremref{##1} Restated}%
 \begin{rep@theorem}}%
 {\end{rep@theorem}}}
\makeatother
\makeatletter
\newtheorem*{rep@lemma}{\rep@title}
\newcommand{\newreplemma}[2]{%
\newenvironment{rep#1}[1]{%
 \def\rep@title{\lemmaref{##1} Restated}%
 \begin{rep@lemma}}%
 {\end{rep@lemma}}}
\makeatother
\newtheorem{theorem}{Theorem}
\newreptheorem{theorem}{Theorem}

\newtheorem{definition}{Definition}
\newtheorem{lemma}{Lemma}
\newreplemma{lemma}{Lemma}

\newtheorem{claim}{Claim}

\newcommand{\namedref}[2]{\texorpdfstring{\hyperref[#2]{#1~\ref*{#2}}}{#1~\ref*{#2}}\xspace}
\newcommand{\lemmaref}[1]{\namedref{Lemma}{lem:#1}}
\newcommand{\theoremref}[1]{\namedref{Theorem}{thm:#1}}
\newcommand{\claimref}[1]{\namedref{Claim}{clm:#1}}

\newcommand{\figureref}[1]{\namedref{Figure}{fig:#1}}
\newcommand{\equationref}[1]{\namedref{Equation}{eq:#1}}
\newcommand{\inequalityref}[1]{\namedref{Inequality}{ineq:#1}}

\newcommand{\sectionref}[1]{\namedref{Section}{sec:#1}}
\newcommand{\appendixref}[1]{\namedref{Appendix}{app:#1}}

\newcommand{\ie}{\text{i.e.}\xspace}

\newcommand{\ceil}[1]{\ensuremath{\left\lceil{#1}\right\rceil}\xspace}
\newcommand{\floor}[1]{\ensuremath{\left\lfloor{#1}\right\rfloor}\xspace}
\newcommand{\abs}[1]{\ensuremath{\left\vert{#1}\right\vert}\xspace}

\newcommand{\eps}[0]{\ensuremath{\varepsilon}}
\let\epsilon\eps

\newcommand{\cA}{\ensuremath{{\mathcal A}}\xspace}
\newcommand{\cB}{\ensuremath{{\mathcal B}}\xspace}

\newcommand{\cD}{\ensuremath{{\mathcal D}}\xspace}

\newcommand{\cF}{\ensuremath{{\mathcal F}}\xspace}

\newcommand{\cP}{\ensuremath{{\mathcal P}}\xspace}

\newcommand{\cR}{\ensuremath{{\mathcal R}}\xspace}
\newcommand{\cS}{\ensuremath{{\mathcal S}}\xspace}

\newcommand{\cX}{\ensuremath{{\mathcal X}}\xspace}

\newcommand{\bbE}{\ensuremath{{\mathbb E}}\xspace}

\newcommand{\bbR}{\ensuremath{{\mathbb R}}\xspace}

\newcommand{\fR}{\ensuremath{{\mathfrak R}}\xspace}

\newcommand{\defeq}[0]{\ensuremath{\;{\vcentcolon=}\;}\xspace}

\newcommand{\E}[0]{\mathop{\bbE}\xspace}

\newcommand{\iid}[0]{\text{i.i.d.}\xspace}

\DeclareMathOperator{\tr}{Tr}

\newcommand{\mat}[1]{\boldsymbol{#1}}
\renewcommand{\vec}[1]{\boldsymbol{\mathrm{#1}}}
\newcommand{\vecalt}[1]{\boldsymbol{#1}}

\newcommand{\normof}[1]{\|#1\|}
\newcommand{\normoflr}[1]{\left|\hspace{-0.05cm}\left|#1\right|\hspace{-0.05cm}\right|}

\newcommand{\sbmat}[1]{\left[\begin{smallmatrix} #1 \end{smallmatrix}\right]}

\newcommand{\mK}{\ensuremath{\mat{K}}\xspace}

\newcommand{\vr}{\ensuremath{\vec{r}}\xspace}

\newcommand{\vw}{\ensuremath{\vec{w}}\xspace}
\newcommand{\vx}{\ensuremath{\vec{x}}\xspace}

\newcommand{\vsigma}{\ensuremath{\vecalt{\sigma}}\xspace}
\newcommand{\valpha}{\ensuremath{\vecalt{\alpha}}\xspace}
\newcommand{\vxi}{\ensuremath{\vecalt{\xi}}\xspace}
\newcommand{\vphi}{\ensuremath{\vecalt{\phi}}\xspace}

\DeclareMathOperator*{\kerdual}{DualSVM}

\newcommand{\mKt}[0]{\ensuremath{\tilde{\mK}}\xspace}
\newcommand{\ksigma}[0]{\ensuremath{k_{_\Sigma}}\xspace}

\newcommand{\news}[1]{#1}

\usepackage[top=1in,bottom=1in,left=1in,right=1in]{geometry}

\title{Optimality Implies Kernel Sum Classifiers are Statistically Efficient}

\author{
  Raphael Arkady Meyer\\
  Department of Computer Science\\
  Purdue University\\
  \texttt{meyer219@purdue.edu}
  \and
  Jean Honorio\\
  Department of Computer Science\\
  Purdue University\\
  \texttt{jhonorio@purdue.edu}
}

\begin{document}

\maketitle

\begin{abstract}
We propose a novel combination of optimization tools with learning theory bounds in order to analyze the sample complexity of \emph{optimal} kernel sum classifiers.
This contrasts the typical learning theoretic results which hold for all (potentially suboptimal) classifiers.
Our work also justifies assumptions made in prior work on multiple kernel learning.
As a byproduct of our analysis, we also provide a new form of Rademacher complexity for hypothesis classes containing only optimal classifiers.

\end{abstract}

\section{Introduction}
\label{sec:intro}

Classification is a fundamental task in machine learning \cite{shalev2014understanding,daume2012course,friedman2001elements}.
Kernel methods allow classifiers to learn powerful nonlinear relationships \cite{shawe2004kernel,balcan2006kernels}.
Optimization tools allow these methods to learn efficiently \cite{soentpiet1999advances}.
Under mild assumptions, kernels guarantee that learned models generalize well \cite{bartlett2002rademacher}.
However, the overall quality of these models still depends heavily on the choice of kernel.
To compensate for this, prior work considers learning how to linearly combine a set of arbitrary kernels into a good data-dependent kernel \cite{sonnenburg2006large,gonen2011multiple,bach2004multiple}.

It is known that if the learned linear combination of kernels is well behaved, then the kernel classifier generalizes well \cite{cortes2010generalization,cortes20092,argyriou2005learning}.
We extend this body of work by proving that if our classifier is optimal, then the linear combination of kernels is well behaved.
This optimality assumption is well justified because many common machine learning problems are solved using optimization algorithms.
For instance, in this paper we consider binary classification with Kernel Support Vector Machines (SVM), which are computed by solving a quadratic programming problem.
Specifically, we bound the sample complexity of kernel classifiers in two regimes.
In the first, we are forced to classify using the sum of a set of kernels.
In the second, we choose which kernels we include in our summation.

There exists substantial prior work considering learning kernels.
From the computational perspective, several theoretically sound and experimentally efficient algorithms are known \cite{cortes2009learning,cortes20092,kivinen2004online,sinha2016learning,duvenaud2013structure}.
Much of this work relies on optimization tools such as quadratic programs \cite{chen2009learning}, sometimes specifically considering Kernel SVM \cite{srebro2006learning}.
This motivates our focus on optimal classifiers for multiple kernel learning.
The literature on sample complexity for these problems always assumes that the learned combination of kernels is well behaved \cite{cortes2009new,cortes2012algorithms,cortes2013learning,srebro2006learning,sinha2016learning}.
That is, the prior work assumes that the weighted sum of kernel matrices \(\mKt_{_\Sigma}\) is paired with a vector \(\valpha_{_\Sigma}\) such that \(\valpha_{_\Sigma}^\intercal\mKt_{_\Sigma}\valpha_{_\Sigma}\leq C^2\) for some constant \(C\).
It is \news{unclear} how \(C\) depends on the structure or \news{number} of base kernels.
Our work provides bounds that explains this relationship for optimal classifiers.
Additionally, Rademacher complexity is typically used to control the generalization error over all possible (not necessarily optimal) estimators \cite{bartlett2002rademacher,koltchinskii2002empirical,kakade2009complexity}.
We differ from this approach by bounding the Rademacher complexity for only optimal estimators.
We are not aware of any prior work that explores such bounds.

\textbf{Contributions.~}
Our results start with a core technical theorem, which is then applied to two novel hypothesis classes.
\begin{itemize}
	\item
	We first show that the optimal solution to the Kernel SVM problem using the sum of \(m\) kernels is well behaved.
	That is, we consider the given kernel matrices \(\mKt_1,\ldots,\mKt_m\) and the corresponding Dual Kernel SVM solution vectors \(\valpha_1,\ldots,\valpha_m\), as well as the sum of these kernel matrices \(\mKt_{_\Sigma}\) and its Dual Kernel SVM solution vector \(\valpha_{_\Sigma}\).
	Using Karush-Kuhn-Tucker (KKT) optimality conditions, we prove that \(\valpha_{_\Sigma}^\intercal\mKt_{_\Sigma}\valpha_{_\Sigma}\leq 3m^{-0.58} B^2\) provided that all base kernels fulfill \(\valpha_t^\intercal\mKt_t\valpha_t\leq B^2\) for some constant \(B\).
	We are not aware of any existing bounds of this kind, and we provide Rademacher complexity analysis to leverage this result.
	Note that the previous bounds for the Rademacher complexity in multiple kernel learning assumes that \(\valpha_{_\Sigma}^\intercal\mKt_{_\Sigma}\valpha_{_\Sigma}\) is bounded.
\end{itemize}
We provide Rademacher complexity bounds for two novel hypothesis classes.
As opposed to traditional Rademacher bounds, our hypothesis classes only contain optimal classifiers.
The traditional analysis when using a \textit{single} kernel provides an empirical Rademacher complexity bound of \(O(\frac{BR}{\sqrt n})\), where \(n\) is the number of samples and \(k_t(\vx_i,\vx_i)\leq R^2\) bounds the radius of the samples in every feature space \cite{bartlett2002rademacher}.
\begin{itemize}
	\item
	\textbf{Kernel Sums:~}In the first set, Kernel SVM is required to use the sum of all \(m\) kernels.
	We show that the empirical Rademacher complexity is bounded by \(O(\frac{BR}{\sqrt n} m^{0.208})\).
	\item
	\textbf{Kernel Subsets:~}In the second set, Kernel SVM is allowed to use the sum of any subset of the \(m\) kernels.
	The classical analysis in this setting would pay a multiplicative factor of \(2^{m-1}\).
	The approach we use instead only pays with a factor of \(\sqrt{\ln(m)}\).
	We prove that the empirical Rademacher complexity is bounded by \(O(\frac{BR\sqrt{\ln(m)}}{\sqrt n} m^{0.208})\).
\end{itemize}
Note that these Rademacher bounds compare naturally to the traditional single kernel bound.
If we use a sum of \(m\) kernels instead of just \news{one} kernel, then we \news{pay a multiplicative factor} of \(m^{0.208}\).
If we use any subset of kernels, we only pay \news{an extra factor of} \(\sqrt{\ln(m)}\).
\news{Thus, in this work, we show} that optimization tools such as KKT conditions are useful in the analysis of statistical bounds.
These optimization bounds are leveraged by learning theoretic tools such as Rademacher complexity, as seen in the second and third bullet points.
Overall, we obtain new bounds with natural assumptions that connect the existing literature on optimization and learning theory in a novel fashion.
Additionally, these bounds justify assumptions made in the existing literature.

\section{Preliminaries}
\label{sec:prelim}

Let \(\cS=\{(\vx_1,y_1),\ldots,(\vx_n,y_n)\}\) denote a dataset of \(n\) \iid samples from some distribution \(\cD\), where \(\vx_i\in\cX\) and \(y_i\in\{-1,1\}\) for some \(\cX\).
Let \(\normof{\cdot}_2\) denote the \(\ell_2\) vector norm and \(\normof{\cdot}_1\) denote the \(\ell_1\) vector norm.
Let \([n]\defeq\{1,\ldots,n\}\) for any natural number \(n\).

Let \(k:\cX\times\cX\rightarrow\bbR\) denote a kernel function.
In this paper, we assume that all kernels fulfill \(\abs{k(\vx,\vx)}<\infty\) for all \(\vx\in\cX\).
We consider being given a set of kernels \(k_1,\ldots,k_m\).
Let \(k_{_\Sigma}(\cdot,\cdot)\defeq \sum_{t=1}^{m}k_t(\cdot,\cdot)\) denote the sum of the \(m\) kernels.
The above notation will be useful when learning with kernel sums.
Let \(\cP\subseteq[m]\).
Then define \(k_{_\cP}(\cdot,\cdot)\defeq\sum_{t\in \cP} k_t(\cdot,\cdot)\) as the sum of kernels as described by \(\cP\).
The latter notation will be useful when learning kernel subsets.

Given a dataset \cS and a kernel \(k_t\), we can build the corresponding kernel matrix \(\mK_t\in\bbR^{n \times n}\), where \([\mK_t]_{i,j} \defeq k_t(\vx_i,\vx_j)\).
Further, we can build the \textit{labeled kernel matrix} \(\mKt_t\), defined elementwise as \([\mKt_t]_{i,j}\defeq y_iy_jk_t(\vx_i,\vx_j)\).
To simplify notation, all our results use labeled kernel matrices instead of standard kernel matrices.

\subsection{Separable SVM}

We now present optimal kernel classification, first in the separable case.
\begin{definition}[Primal Kernel SVM]
Given a dataset \(\cS = \{(\vx_1,y_1),\ldots,(\vx_n,y_n)\}\) and a feature map \(\vphi:\cX\mapsto\bbR^{d}\), the Primal Kernel SVM problem is equivalent to the following optimization problem:
\begin{align*}
	\min_{\vw} &~ \frac12 \normof{\vw}_2^2 \\
	\text{s.t.} &~ 1 - y_i \vw^\intercal\vphi(\vx_i) \leq 0 ~~ \forall i \in [n]
\end{align*}
\end{definition}
We will mainly look at the corresponding dual problem:
\begin{definition}[Dual Kernel SVM]
\label{def:dual-svm}
Given a dataset \(\cS=\{(\vx_1,y_1),\ldots,(\vx_n,y_n)\}\) and a kernel function \(k(\cdot,\cdot)\) with associated labeled kernel matrix \mKt, the Dual Kernel SVM problem is equivalent to the following optimization problem:
\begin{align*}
	\max_{\valpha} &~ \normof{\valpha}_1 - \frac12 \valpha^\intercal \mKt \valpha \\
	\text{s.t.} &~ \alpha_i \geq 0 ~~ \forall i \in [n]
\end{align*}
\end{definition}
Since the dual optimization problem is defined entirely in terms of \mKt, we can denote the optimal \valpha as a function of the labeled kernel matrix.
We write this as \(\valpha = \kerdual(\mKt)\).

Recall that Karush-Kuhn-Tucker (KKT) conditions are necessary and sufficient for optimality in convex optimization problems \cite{boyd2004convex}.
We can express the KKT conditions of the Primal Kernel SVM as follows:

\emph{Primal Feasibility:}
\begin{align}
\label{eq:primal-feasible}
	1-y_i\vw^\intercal\vphi(\vx_i) \leq 0 ~~ \forall i \in [n]
\end{align}
\emph{Stationarity:}
\begin{align}
\label{eq:stationarity}
	\vw=\sum_{i=1}^{n}\alpha_iy_i\vphi(\vx_i)
\end{align}
\emph{Dual Feasibility:}
\begin{align}
\label{eq:dual-feasibility}
\alpha_i\geq0 ~~ \forall i \in [n]
\end{align}
\emph{Complementary Slackness:}
\begin{align}
\label{eq:complementary-slackness}
\alpha_i(1-y_i\vw^\intercal\vphi(\vx_i))=0
\end{align}

The above KKT conditions will be used with learning theoretic tools in order to provide novel generalization bounds.

\subsection{Non-separable SVM}
The primal and dual SVMs above assume that the given kernel is able to separate the data perfectly.
Since this is not always the case, we also consider non-separable data using \(\ell_2\) slack variables:
\begin{definition}[Primal Kernel SVM with Slack Variables]
Given \(C>0\), a dataset \(\cS = \{(\vx_1,y_1),\ldots,(\vx_n,y_n)\}\), and a feature map \(\vphi:\cX\mapsto\bbR^{d}\), the Primal Kernel SVM problem is equivalent to the following optimization problem:
\begin{align*}
	\min_{\vw} &~ \frac12 \normof{\vw}_2^2 + \frac{C}{2} \normof{\vxi}_2^2 \\
	\text{s.t.} &~ 1 - y_i \vw^\intercal\vphi(\vx_i) \leq \xi_i ~~ \forall i \in [n] \\
	&~ \xi_i \geq 0 ~~ \forall i \in [n]
\end{align*}
\end{definition}
\begin{definition}[Dual Kernel SVM with Slack Variables]
Given \(C>0\), a dataset \(\cS=\{(\vx_1,y_1),\ldots,(\vx_n,y_n)\}\) and a kernel function \(k(\cdot,\cdot)\) with associated labeled kernel matrix \mKt, the Dual Kernel SVM problem is equivalent to the following optimization problem:
\begin{align*}
	\max_{\valpha,\vxi} &~ \normof{\valpha}_1 - \frac12 \valpha^\intercal \mKt \valpha - \frac12 \normof{\vxi}_2^2 \\
	\text{s.t.} &~ 0 \leq \alpha_i \leq C\xi_i ~~ \forall i \in [n]
\end{align*}
\end{definition}
We denote the solution to the Dual SVM with Slack Variables using parameter \(C\) as \(\valpha=\kerdual_C(\mKt)\).

\subsection{Rademacher Complexity for Kernels}
We use Rademacher complexity to bound the sample complexity of kernel methods.
The empirical Rademacher complexity of a hypothesis class \cF with dataset \cS is defined as
\[
	\hat\fR_{\cS}(\cF) = \E_{\vsigma}\left[\sup_{h\in\cF}\left(\frac1n\sum_{i=1}^{n}\sigma_ih(\vx_i)\right)\right]
\]
where \(\vsigma\in\{-1,+1\}^n\) is a vector of Rademacher variables.
Bartlett and Mendelson introduced the analysis of sample complexity for kernel methods via Rademacher complexity when using one kernel \cite{bartlett2002rademacher}.
Bartlett and Mendelson considered the following hypothesis class of representer theorem functions:
\begin{align}
	\label{eq:bartlett-hypothesis}
	\cF \defeq
	\left\{
		\vx \mapsto \sum_{i=1}^{n} \alpha_i k(\vx,\vx_i),
		\vx_i\in\cX,
		\valpha^\intercal \mK \valpha \leq B^2
	\right\}
\end{align}
Each element of \cF is defined in terms of a dataset \(\vx_1,\ldots,\vx_n\) and an \valpha vector.
Bartlett and Mendelson showed that the probability of misclassification is bounded by the empirical risk of misclassification with a \(\gamma\)-Lipschitz loss plus a Rademacher term:
\begin{theorem}[Theorem 22 from \cite{bartlett2002rademacher}]
\label{thm:bartlett-lipschitz-risk}
Fix \(n \geq 0\), \(\gamma\in(0,1)\), and \(\delta\in(0,1)\).

Let \(\cS=\{(\vx_1,y_1),\ldots,(\vx_n,y_n)\}\) be a dataset of \(n\) \iid samples from \cD.
That is, let \(\cS\sim\cD^n\).
Define the \(\gamma\)-Lipschitz Loss function
\[
	\psi(x)\defeq
	\begin{cases}
		1 & x < 0 \\
		1-\frac x\gamma & 0 \leq x \leq \gamma \\
		0 & x > \gamma
	\end{cases}
\]
Let \(\eps\defeq(\frac8\gamma + 1)\sqrt{\frac{\ln(4/\delta)}{2n}}\).
Then with probability at least \(1-\delta\) over the choice of \cS, for all \(f\in\cF\) we have
\[
	\Pr_{(\vx,y)\sim\cD}[yf(x)\leq0]
	\leq
	\frac1n\sum_{i=1}^{n}\psi(y_i f(\vx_i)) + \frac{2}{\gamma} \hat\fR_\cS(\cF) + \eps
\]
\end{theorem}
In this paper, our interest is in bounding this \(\hat\fR_\cS(\cF)\) term under reasonable assumptions on \cF.
We specifically consider two hypothesis classes defined over a set of \(m\) kernels.
First, we consider optimal kernel sum classification, where we must use the sum of all \(m\) given kernels:
\begin{align}
	\label{eq:cf-sum}
	\cF_{_{\Sigma}}
	\defeq
	\bigg\{
		&\vx \mapsto \sum_{i=1}^{n} \alpha_i y_i k_{_\Sigma}(\vx,\vx_i),
		\vx_i\in\cX,
		y_i\in\{-1,1\},\nonumber\\
		&\hspace{0.5cm} \valpha = \kerdual(\mKt_{_\Sigma}), \nonumber\\
		&\hspace{0.5cm} \valpha_t^\intercal \mKt_t \valpha_t \leq B^2 ~~ \forall t\in[m]
	\bigg\}
\end{align}
Second, we consider optimal kernel subset classification, where we \news{are allowed to use the sum of} any subset of the \(m\) given kernels:
\begin{align}
	\label{eq:cf-subset}
	\cF_{_{\cP}}
	\defeq
	\bigg\{
		&\vx \mapsto \sum_{i=1}^{n} \alpha_i y_i k_{_\cP}(\vx,\vx_i),
		\vx_i\in\cX,
		y_i\in\{-1,1\},
		\nonumber\\
		&\hspace{0.5cm}
		\cP\subseteq[m],
		\valpha = \kerdual(\mKt_{_\cP}),
		\nonumber\\
		&\hspace{0.5cm}
		\valpha_t^\intercal \mKt_t \valpha_t \leq B^2 ~~ \forall t\in[m]
	\bigg\}
\end{align}

Note that \(y_i\) is not present in Bartlett and Mendelson's hypothesis class in \equationref{bartlett-hypothesis}, but it is in \equationref{cf-sum} and \equationref{cf-subset}.
Regardless, \(\cF_{_\Sigma}\) and \(\cF_{_\cP}\) do not allow for a more general set of \(\valpha\) vectors.
This is because \(\alpha_i\) is allowed to be both positive and negative in \(\cF\).
However, in \(\cF_{_\Sigma}\) and \(\cF_{_\cP}\), \(\valpha\) is a dual optimal vector.
Dual Feasibility implies \(\alpha_i\geq0\).
Thus, by explicitly mentioning \(y_i\) in the definitions of \(\cF_{_\Sigma}\) and \(\cF_{_\cP}\), we are stating that \(\alpha_i\) in \(\cF\) equals \(\alpha_iy_i\) in \(\cF_{_\Sigma}\) and \(\cF_{_\cP}\).

Initial Rademacher complexity bounds for learning with a single kernel assume that \(\valpha^\intercal \mKt \valpha \leq B^2\) \cite{bartlett2002rademacher}.
Previous lines of work on multiple kernel learning then assume that \(\valpha_{_\Sigma}^\intercal \mKt_{_\Sigma} \valpha_{_\Sigma} \leq C^2\) for some constant \(C\) \cite{cortes2010generalization,cortes20092,sinha2016learning,srebro2006learning}.
We are interested in proving what values of \(C\) are reasonable.
To achieve this, we assume that \(\valpha_t^\intercal \mKt_t \valpha_t \leq B^2\) for all base kernels and show that \(\valpha_{_\Sigma}^\intercal \mKt_{_\Sigma} \valpha_{_\Sigma}\) is indeed bounded.

In \sectionref{svm-bounds}, we leverage our assumption that \(\valpha\) is optimal to build this bound on \(\valpha_{_\Sigma}^\intercal\mKt_{_\Sigma}\valpha_{_\Sigma}\).
In \sectionref{rademacher-bounds}, we demonstrate how our bound can augment existing techniques for bounding the Rademacher complexities of \equationref{cf-sum} and \equationref{cf-subset}.
That is, we bound \(\hat\fR_{_\cS}(\cF_{_\Sigma})\) and \(\hat\fR_{_\cS}(\cF_{_\cP})\).

\section{SVM Bounds for Sums of Kernels}

\label{sec:svm-bounds}
In this section, we leverage KKT conditions and SVM optimality to control the value of \(\valpha^\intercal\mKt\valpha\) as the number of kernels grows.
To start, we consider a single kernel \(k\):
\begin{lemma}
\label{lem:one-kernel}
Let \(\valpha=\kerdual(\mKt)\) for some kernel matrix \mKt.
Then \(\normof{\valpha}_1 = \valpha^\intercal\mKt\valpha\).
\end{lemma}
\begin{proof}
This proof follows from the KKT conditions provided in \sectionref{prelim}.
We start by substituting Stationarity (\equationref{stationarity}) into Complementary Slackness (\equationref{complementary-slackness}).
For all \(i\in[n]\),
\begin{align*}
	0 &= \alpha_i (1 - y_i \vw^\intercal \vphi(\vx_i)) \\
	0 &= \alpha_i \left(1 - \left(\sum\nolimits_{j=1}^{n} \alpha_jy_j\vphi(\vx_j)\right)^\intercal y_i \vphi(\vx_i)\right) \\
	0 &= \alpha_i \left(1 - \sum\nolimits_{j=1}^{n} \alpha_jy_iy_j\vphi(\vx_j)^\intercal\vphi(\vx_i)\right) \\
	0 &= \alpha_i - \sum_{j=1}^{n} \alpha_i\alpha_jy_iy_j\vphi(\vx_j)^\intercal\vphi(\vx_i) \\
	\alpha_i &= \sum_{j=1}^{n} \alpha_i\alpha_j[\mKt]_{i,j}
\end{align*}
We can then take the sum of both sides over all \(i\):
\begin{align}
	\sum_{i=1}^n\alpha_i &= \sum_{i=1}^n\sum_{j=1}^{n} \alpha_i\alpha_j[\mKt]_{i,j} \nonumber \\
	\label{eq:one-kernel-equals}
	\normof{\valpha}_1 &= \valpha^\intercal\mKt\valpha
\end{align}
Note that \(\sum_{i=1}^n\alpha_i=\normof{\valpha}_1\) since Dual Feasibility (\equationref{dual-feasibility}) tells us that \(\alpha_i \geq 0\).
\end{proof}

\lemmaref{one-kernel} is mathematically \news{meaningful since} at the optimal point \(\valpha\), the Dual SVM takes objective value exactly equal to \(\frac12 \valpha^\intercal\mKt\valpha\).
This connects the objective value at the optimal point to the term we want to control.
With this in mind, we now move on to consider having two kernels \(k_1\) and \(k_2\).
\begin{theorem}
\label{thm:two-kernels}
Let \(\cS=\{(\vx_1,y_1),\ldots,(\vx_n,y_n)\}\) be a dataset.
Let \(k_1,k_2\) be kernel functions.
Define \(k_{1+2}(\cdot,\cdot)\defeq k_1(\cdot, \cdot) + k_2(\cdot, \cdot)\).
Let \(\mKt_1,\mKt_2,\mKt_{1+2}\) be their labeled kernel matrices and \(\valpha_1,\valpha_2\valpha_{1+2}\) be the corresponding Dual SVM solutions.
Then we have
\[
	\valpha_{1+2}^\intercal\mKt_{1+2}\valpha_{1+2}
	\leq
	\frac13(\valpha_1^\intercal\mKt_1\valpha_1 + \valpha_2^\intercal\mKt_2\valpha_2)
\]
Furthermore,
\begin{align}
	\label{ineq:two-thirds}
	\valpha_{1+2}^\intercal\mKt_{1+2}\valpha_{1+2}
	\leq
	\frac23 \max\{\valpha_1^\intercal\mKt_1\valpha_1, \valpha_2^\intercal\mKt_2\valpha_2\}
\end{align}
\end{theorem}

\begin{proof}
First recall that if \(\valpha=\kerdual(\mKt)\), then for all other dual feasible \(\valpha'\),
\begin{align}
	\label{ineq:app-dual-optimality}
	\normof{\valpha'}_1 - \frac12 {\valpha'}^\intercal\mKt\valpha'
	\leq
	\normof{\valpha}_1 - \frac12 \valpha^\intercal\mKt\valpha
\end{align}
Also note that \(\mKt_{1+2}=\mKt_1+\mKt_2\).
We start the proof by looking at the Dual SVM objective for \(k_{1+2}\), and distributing over the labeled kernel matrices:
\begin{align*}
	\normof{\valpha_{1+2}}_1 &- \frac12 \valpha_{1+2}^\intercal\mKt_{1+2}\valpha_{1+2} \\
	&= \normof{\valpha_{1+2}}_1 - \frac12 \valpha_{1+2}^\intercal(\mKt_1+\mKt_2)\valpha_{1+2} \\
	&= \normof{\valpha_{1+2}}_1 - \frac12 \valpha_{1+2}^\intercal\mKt_1\valpha_{1+2} \\
	&\phantom{=} ~~ \phantom{\normof{\valpha_{1+2}}_1} - \frac12 \valpha_{1+2}^\intercal\mKt_2\valpha_{1+2} \\
\intertext{We now introduce an extra \(\normof{\valpha_{1+2}}_1\) term by adding zero. This allows us to form two expressions that look like Dual SVM Objectives.}
	\normof{\valpha_{1+2}}_1 &- \frac12 \valpha_{1+2}^\intercal\mKt_{1+2}\valpha_{1+2} \\
	&= \normof{\valpha_{1+2}}_1 - \frac12 \valpha_{1+2}^\intercal\mKt_1\valpha_{1+2} \\
	&\phantom{=} ~~ \phantom{\normof{\valpha_{1+2}}_1} - \frac12 \valpha_{1+2}^\intercal\mKt_2\valpha_{1+2} \\
	&\phantom{=} ~~ \phantom{\normof{\valpha_{1+2}}_1} + \normof{\valpha_{1+2}}_1 - \normof{\valpha_{1+2}}_1 \\
	&= \left(\normof{\valpha_{1+2}}_1 - \frac12 \valpha_{1+2}^\intercal\mKt_1\valpha_{1+2}\right) \\
	&\phantom{=} ~~ + \left(\normof{\valpha_{1+2}}_1 - \frac12 \valpha_{1+2}^\intercal\mKt_2\valpha_{1+2}\right) \\
	&\phantom{=} ~~ - \normof{\valpha_{1+2}}_1 \\
\intertext{We then apply \inequalityref{app-dual-optimality} to both of these parentheses:}
	\normof{\valpha_{1+2}}_1 &- \frac12 \valpha_{1+2}^\intercal\mKt_{1+2}\valpha_{1+2} \\
	&\leq \left(\normof{\valpha_1}_1 - \frac12 \valpha_1^\intercal\mKt_1\valpha_1\right) \\
	&\phantom{=} ~~ + \left(\normof{\valpha_2}_1 - \frac12 \valpha_2^\intercal\mKt_2\valpha_2\right) \\
	&\phantom{=} ~~ - \normof{\valpha_{1+2}}_1
\end{align*}
Reorganizing the above equation, we get
\begin{align}
	2\normof{\valpha_{1+2}}_1 - \frac12 \valpha_{1+2}^\intercal&\mKt_{1+2}\valpha_{1+2} \nonumber \\
	&\leq \left(\normof{\valpha_1}_1 - \frac12 \valpha_1^\intercal\mKt_1\valpha_1\right) \nonumber \\
\label{ineq:two-kernels-part}
	&\phantom{=} ~~ + \left(\normof{\valpha_2}_1 - \frac12 \valpha_2^\intercal\mKt_2\valpha_2\right)
\end{align}
Next, we use \lemmaref{one-kernel} to simplify all three expression that remain:
\begin{itemize}
	\item \(2\normof{\valpha_{1+2}}_1 - \frac12 \valpha_{1+2}^\intercal\mKt_{1+2}\valpha_{1+2} = \frac32\normof{\valpha_{1+2}}_1\)
	\item \(\normof{\valpha_1}_1 - \frac12 \valpha_1^\intercal\mKt_1\valpha_1 = \frac12\normof{\valpha_1}_1\)
	\item \(\normof{\valpha_2}_1 - \frac12 \valpha_2^\intercal\mKt_2\valpha_2 = \frac12\normof{\valpha_2}_1\)
\end{itemize}
Returning to our bound from \inequalityref{two-kernels-part}, we have
\begin{align}
\label{ineq:two-kernels-final}
	\frac32\normof{\valpha_{1+2}}_1 \leq \frac12\normof{\valpha_1}_1 + \frac12\normof{\valpha_2}_1
\end{align}
Once we rearrange the constants in \inequalityref{two-kernels-final}, we complete the proof.
\end{proof}
The constant of \(\frac23\) in \inequalityref{two-thirds} is \news{advantageous}.
Since this ratio is below 1, we can recursively apply this theorem to get a vanishing fraction.
As \(m\) increases, we should now expect \(\valpha_{_\Sigma}^\intercal\mKt_{_\Sigma}\valpha_{_\Sigma}\) to decrease.
We formalize this notion in the following theorem, where we consider using the sum of \(m\) kernels.

\begin{theorem}
\label{thm:svm-many-kernels}
Let \(\cS=\{(\vx_1,y_1),\ldots,(\vx_n,y_n)\}\) be a dataset.
Let \(k_1,k_2,\ldots,k_m\) be kernel functions.
Define \(k_{_\Sigma}(\cdot,\cdot)\defeq \sum_{t=1}^{m} k_t(\cdot, \cdot)\).
Let \(\mKt_1,\ldots,\mKt_m,\mKt_{_\Sigma}\) be their labeled kernel matrices and \(\valpha_1,\ldots,\valpha_m,\valpha_{_\Sigma}\) be the corresponding Dual SVM solutions.
Then we have
\[
	\valpha_{_\Sigma}^\intercal\mKt_{_\Sigma}\valpha_{_\Sigma}
	\leq
	3 m^{-\log_2(3)} \sum_{t=1}^{m}\valpha_t^\intercal\mKt_t\valpha_t
\]
Furthermore,
\[
	\valpha_{_\Sigma}^\intercal\mKt_{_\Sigma}\valpha_{_\Sigma}
	\leq
	3 m^{-\log_2(\nicefrac32)} \max_{t\in[m]} \valpha_t^\intercal\mKt_t\valpha_t
\]
In the special case that \(m\) is a power of \(2\), we have
\begin{align*}
	\valpha_{_\Sigma}^\intercal\mKt_{_\Sigma}\valpha_{_\Sigma}
	&\leq
	m^{-\log_2(3)} \sum_{t=1}^{m}\valpha_t^\intercal\mKt_t\valpha_t \\
	&\leq
	m^{-\log_2(\nicefrac32)} \max_{t\in[m]} \valpha_t^\intercal\mKt_t\valpha_t \\
\end{align*}
\end{theorem}
\begin{proof}[Proof sketch]
We provide an intuitive proof for \(m=8\).
The full proof is in \appendixref{svm-many-kernels}.

Since \(m\) is a power of two, we can label each of the base kernels with length \(\ell=\log_2(m)=3\) bitstrings:
\begin{figure}[H]
\centering
\scalebox{0.7}
{
\begin{tikzpicture}
\node (k000) at (0   ,0){\(k_{000}\)};
\node (k001) at (1.5 ,0){\(k_{001}\)};
\node (k010) at (3   ,0){\(k_{010}\)};
\node (k011) at (4.5 ,0){\(k_{011}\)};
\node (k100) at (6   ,0){\(k_{100}\)};
\node (k101) at (7.5 ,0){\(k_{101}\)};
\node (k110) at (9   ,0){\(k_{110}\)};
\node (k111) at (10.5,0){\(k_{111}\)};
\end{tikzpicture}
}
\end{figure}
Then, for each pair of kernels that differ only in the last digit, define a new kernel as their sum.
For instance, define \(k_{10}(\cdot,\cdot)\defeq k_{100}(\cdot,\cdot) + k_{101}(\cdot,\cdot)\).
Repeat this process all the way to the root node.
\begin{figure}[H]
\centering
\scalebox{0.7}
{
\begin{tikzpicture}
\node (k000) at (0   ,0){\(k_{000}\)};
\node (k001) at (1.5 ,0){\(k_{001}\)};
\node (k010) at (3   ,0){\(k_{010}\)};
\node (k011) at (4.5 ,0){\(k_{011}\)};
\node (k100) at (6   ,0){\(k_{100}\)};
\node (k101) at (7.5 ,0){\(k_{101}\)};
\node (k110) at (9   ,0){\(k_{110}\)};
\node (k111) at (10.5,0){\(k_{111}\)};

\node (k00) at (0.75,1.5){\(k_{00}\)};
\node (k01) at (3.75,1.5){\(k_{01}\)};
\node (k10) at (6.75,1.5){\(k_{10}\)};
\node (k11) at (9.75,1.5){\(k_{11}\)};

\draw (k000) -- (k00);
\draw (k001) -- (k00);
\draw (k010) -- (k01);
\draw (k011) -- (k01);
\draw (k100) -- (k10);
\draw (k101) -- (k10);
\draw (k110) -- (k11);
\draw (k111) -- (k11);

\node (k0) at (2.25,3){\(k_{0}\)};
\node (k1) at (8.25,3){\(k_{1}\)};

\draw (k00) -- (k0);
\draw (k01) -- (k0);
\draw (k10) -- (k1);
\draw (k11) -- (k1);

\node (ks) at (5.25,4.5){\(\ksigma\)};

\draw (k0) -- (ks);
\draw (k1) -- (ks);
\end{tikzpicture}
}
\end{figure}

By \theoremref{svm-many-kernels}, we know that
\[
	\valpha_{_\Sigma}^\intercal\mKt_{_\Sigma}\valpha_{_\Sigma}
	\leq
	\frac13(
		\valpha_{0}^\intercal\mKt_{0}\valpha_{0} +
		\valpha_{1}^\intercal\mKt_{1}\valpha_{1}
	)
\]
Going down one level, by applying \theoremref{svm-many-kernels} again, we know that
\[
	\valpha_{0}^\intercal\mKt_{0}\valpha_{0}
	\leq
	\frac13(
		\valpha_{00}^\intercal\mKt_{00}\valpha_{00} +
		\valpha_{01}^\intercal\mKt_{01}\valpha_{01}
	)
\]
Therefore, by similarly applying \theoremref{svm-many-kernels} to \(\valpha_{1}^\intercal\mKt_{1}\valpha_{1}\), we can combine these claims:
\begin{align*}
	\valpha_{_\Sigma}^\intercal\mKt_{_\Sigma}\valpha_{_\Sigma}
	\leq
	\left(\frac13\right)^2 (
		&\valpha_{00}^\intercal\mKt_{00}\valpha_{00} +
		\valpha_{01}^\intercal\mKt_{01}\valpha_{01} + \\
		&\valpha_{10}^\intercal\mKt_{10}\valpha_{10} +
		\valpha_{11}^\intercal\mKt_{11}\valpha_{11}
	)
\end{align*}
We can then continue until all 8 kernels are included:
\[
	\valpha_{_\Sigma}^\intercal\mKt_{_\Sigma}\valpha_{_\Sigma}
	\leq
	\left(\frac13\right)^3 \sum_{t=1}^{m} \valpha_{t}^\intercal\mKt_{t}\valpha_{t}
\]
Note that the exponent of \(\frac13\) is the depth of the tree, equivalent to the length \(\ell\) of our bitstring labels.
In the general case, we have
\begin{align*}
	\valpha_{_\Sigma}^\intercal\mKt_{_\Sigma}\valpha_{_\Sigma}
	& \leq \left(\frac13\right)^{\log_2(m)} \sum_{t=1}^{m} \valpha_{t}^\intercal\mKt_{t}\valpha_{t} \\
	& = m^{-\log_2(3)} \sum_{t=1}^{m} \valpha_{t}^\intercal\mKt_{t}\valpha_{t}
\end{align*}
This completes the analysis if \(m\) is a power of 2.
If we do not have an exact power of two number of kernels, then our tree has depth \(\ell-1\) for some leaves.
Therefore, we place a floor function around \(\log_2(3)\):
\begin{align*}
	\valpha_{_\Sigma}^\intercal\mKt_{_\Sigma}\valpha_{_\Sigma}
	& \leq \left(\frac13\right)^{\floor{\log_2(m)}} \sum_{t=1}^{m} \valpha_{t}^\intercal\mKt_{t}\valpha_{t} \\
	& \leq 3 \left(\frac13\right)^{\log_2(m)} \sum_{t=1}^{m} \valpha_{t}^\intercal\mKt_{t}\valpha_{t} \\
	& = 3 m^{-\log_2(3)} \sum_{t=1}^{m} \valpha_{t}^\intercal\mKt_{t}\valpha_{t}
\end{align*}
To achieve the final result, we bound the summation with
\[\sum_{t=1}^{m}\valpha_t^\intercal\mKt_t\valpha_t \leq m \max_{t\in[m]} \valpha_t^\intercal\mKt_t\valpha_t\]
and simplify the resulting expression.
\end{proof}

We take a moment to reflect on this result.
It has been well established that the generalization error of kernel classifiers depends on \(\valpha_{_\Sigma}^\intercal \mKt_{_\Sigma} \valpha_{_\Sigma}\) \cite{cortes2009new,cortes2012algorithms,cortes2013learning,srebro2006learning,sinha2016learning}.
\theoremref{svm-many-kernels} shows that this term actually decreases in the number of kernels.
In the next section, we show how this theorem translates into generalization error results.

\section{Rademacher Bounds}
\label{sec:rademacher-bounds}
In this section we apply \theoremref{svm-many-kernels} to bound the Rademacher complexity of learning with sums of kernels.
To better parse and understand these bounds, we make two common assumptions:
\begin{itemize}
	\item Each base kernel has a bounded Dual SVM solution:
	 	\[\valpha_t^\intercal\mKt_t\valpha_t\leq B^2 ~~~ \forall t\in[m]\]
 	\item Each vector has a bounded \(\ell_2\) norm in each feature space:
	 	\[k_t(\vx_i,\vx_i)\leq R^2 ~~~ \forall t\in[m],i\in[n]\]
 \end{itemize}
The classical bound in \cite{bartlett2002rademacher} on the Rademacher complexity of kernel functions looks at a \textit{single} kernel, and provides the bound
\[
	\hat\fR_\cS(\cF)\leq \frac{BR}{\sqrt n}
\]
where the hypothesis class \(\cF\) is defined in \equationref{bartlett-hypothesis}.
Our bounds are on the order of \(\frac{BR}{\sqrt n} m^{0.208}\).
That is, when moving from one kernel to many kernels, we pay sublinearly in the number of kernels.

We first see this with our bound on the Rademacher complexity of the kernel sum hypothesis class \(\hat\cR_\cS(\cF_{_\Sigma})\) defined in \equationref{cf-sum}:
\begin{theorem}
\label{thm:rademacher-kernel-sum}
Let \(\cS=\{(\vx_1,y_1),\ldots,(\vx_n,y_n)\}\) be a dataset.
Let \(k_1,\ldots,k_m\) be kernel functions.
Define \(k_{_\Sigma}(\cdot,\cdot)\defeq\sum_{t=1}^{m}k_t(\cdot,\cdot)\).
Let \(\mKt_1,\ldots,\mKt_m,\mKt_{_\Sigma}\) be their labeled kernel matrices and \(\valpha_1,\ldots,\valpha_m,\valpha_{_\Sigma}\) be the corresponding Dual SVM solutions.
Then,
\[
	\hat\fR_\cS(\cF_{_\Sigma})
	\leq
	\frac1n\sqrt{3m^{-\log_2(3)} \left(\sum_{t=1}^{m}\tr[\mKt_t]\right) \sum_{t=1}^{m}\valpha_t^\intercal\mKt_t\valpha_t}
\]
Furthermore, if we assume that \(\valpha_t^\intercal\mKt_t\valpha_t\leq B^2\) and \(k_t(\vx_i,\vx_i)\leq R^2\) for all \(t\in[m]\) and \(i\in[n]\), then we have
\[
	\hat\fR_\cS(\cF_{_\Sigma})
	\leq
	\frac{BR}{\sqrt n} ~ \sqrt{3m^{(1-\log_2(\nicefrac32))}}
	\in
	O\left(\frac{BRm^{0.208}}{\sqrt n}\right)
\]
\end{theorem}
Our proof parallels that of Lemma 22 in \cite{bartlett2002rademacher}, and a full proof is in \appendixref{rademacher-kernel-sum}.
The key difference between Bartlett and Mendelson's proof and ours is the assumption that \(\valpha\) is optimal, allowing us to apply \theoremref{svm-many-kernels}.

Next, we consider learning which kernels to sum.
In this setting, we allow an algorithm to pick any subset of kernels to sum, but require that Kernel SVM is used for prediction.
This is described by the hypothesis class \(\cF_{_\cP}\) defined in \equationref{cf-subset}.
Because the algorithm can pick any arbitrary subset, we are intuitively bounded by the worst risk over all subsets of kernels.
Specifically, \theoremref{rademacher-kernel-sum} suggests that the risk of \(\cF_{_\cP}\) is bounded by the risk of a subset with size \(m\).
That is, the risk of \(\cF_{_\cP}\) is bounded by the risk of using all kernels.
Our next theorem makes this intuition precise, because we only pay an asymptotic factor of \(\sqrt{\ln(m)}\) more when considering all possible subsets of kernels instead of only one subset of kernels.
\begin{theorem}
\label{thm:learn-kernel-bound}
Let \(\cS=\{(\vx_1,y_1),\ldots,(\vx_n,y_n)\}\) be a dataset.
Let \(k_1,\ldots,k_m\) be kernel functions.
Consider any \(\cP\subseteq[m]\).
Define \(k_{_\cP}(\cdot,\cdot)\defeq\sum_{t\in\cP}k_t(\cdot,\cdot)\).
Let \(\mKt_1,\ldots,\mKt_m,\mKt_{_\cP}\) be their labeled kernel matrices and \(\valpha_1,\ldots,\valpha_m,\valpha_{_\cP}\) be the corresponding Dual SVM solutions.
Assume \(k_t(\vx_i,\vx_i) \leq R^2\) and \(\valpha_t^\intercal\mKt_t\valpha_t\leq B^2\) for all \(t\in[m]\) and \(i\in[n]\).
Then,
\begin{align*}
	\hat\fR_\cS(\cF_{_\cP})
	&\leq \frac{BR \sqrt{3e\eta_0 ~ m^{(1-\log_2(\nicefrac{3}{2}))} \ceil{\ln(m)}}}{\sqrt n} \\
	&\in O\left(\frac{BRm^{0.208} \sqrt{\ln(m)}}{\sqrt n}\right)
\end{align*}
where \(\eta_0=\frac{23}{22}\).
\end{theorem}
If we tried to build this bound with the classical analytical method found in Lemma 22 of \cite{bartlett2002rademacher}, we would have to deal with a difficult supremum over the \(2^m\) distinct choices of kernels.
This would inflate the bound by a multiplicative factor of \(\sqrt{2^m}=2^{m-1}\).
However, our proof instead follows that of Theorem 1 in \cite{cortes2009new}.
This more complicated proof method allows us to pay a factor of \(\sqrt{ln(m)}\) to separate supremum over the choice of kernels and the expectation over the \vsigma vector.
This separation allows us to invoke \theoremref{svm-many-kernels}.
However, this proof technique also prevents us from building a claim as general as \theoremref{rademacher-kernel-sum}, instead only providing bounds using \(B\) and \(R\).
The full proof is found in \appendixref{rademacher-kernel-learn}.

\section{Bounds for Non-separable Data}
Recall the Primal and Dual SVMs with \(\ell_2\) slack variables from \sectionref{prelim}.
Now we show that if \(C=\frac12\), then all the other bounds hold using non-separable SVM instead of the separable one.
We achieve this by mirroring \theoremref{two-kernels}, which is used by all other results in this paper.
\begin{theorem}
\label{thm:two-kernels-slack}
Let \(\cS=\{(\vx_1,y_1),\ldots,(\vx_n,y_n)\}\) be a dataset.
Let \(k_1,k_2\) be kernel functions.
Define \(k_{1+2}(\cdot,\cdot)\defeq k_1(\cdot, \cdot) + k_2(\cdot, \cdot)\).
Let \(\mKt_1,\mKt_2,\mKt_{1+2}\) be their labeled kernel matrices and \(\valpha_1,\valpha_2,\valpha_{1+2}\) be the corresponding Dual SVM solutions with parameter \(C=\frac12\).
Then we have
\[
	\valpha_{1+2}^\intercal\mKt_{1+2}\valpha_{1+2}
	\leq
	\frac13(\valpha_1^\intercal\mKt_1\valpha_1 + \valpha_2^\intercal\mKt_2\valpha_2)
\]
Furthermore,
\[
	\valpha_{1+2}^\intercal\mKt_{1+2}\valpha_{1+2}
	\leq
	\frac23 \max\{\valpha_1^\intercal\mKt_1\valpha_1, \valpha_2^\intercal\mKt_2\valpha_2\}
\]
\end{theorem}
The proof mirrors the proof of \theoremref{svm-many-kernels}, except for some careful book keeping for the slack vectors \(\vxi_1\), \(\vxi_2\), and \(\vxi_{1+2}\).
Again, it is the KKT conditions that allow us to bound and compare the \(\vxi\) vectors of the three Dual SVM problems.
The full proof is in \appendixref{proof-two-kernels-slack}.

With \theoremref{two-kernels-slack}, we can reproduce all other results without any changes to the original proofs.
Further, this sort of bound on \(C\) being a constant is common in learning theory literature such as PAC Bayes \cite{mcallester2007}.

\section{Experiment}
\label{sec:experiments}
We show some experimental results that verify our core theorem, \ie \theoremref{svm-many-kernels}.
Our experiment uses 8 fixed kernels from several kernels families.
We have 5 radial basis kernels, 1 linear kernel, 1 polynomial kernel, and 1 cosine kernel.
All our data is generated from a mixture of 4 Gaussians.
Two of the Gaussians generate the positive class while the other 2 generate the negative class.

We generate \(n=300\) samples in \(\bbR^{50}\).
For each of the 8 base kernels, we solve the Dual Kernel SVM problem, and empirically verify that \(\valpha_t^\intercal\mKt_t\valpha_t\leq 320 = B^2\).

Then, we arbitrarily permute the kernel matrices.
We solve the Dual Kernel SVM problem with the first kernel matrix denoted as \(\mKt_{_{\Sigma,1}}\).
Then we solve the SVM with the sum of the first two kernels, denoted as \(\mKt_{_{\Sigma,2}}\), and so on until we sum all 8 kernels.
Let \(\valpha_{_{\Sigma,m}}\) denote the dual solution vector corresponding sum of the first \(m\) of the 8 kernels.
That is,
	\[
		\valpha_{_{\Sigma,m}}
		\defeq
		\kerdual(\mKt_{_{\Sigma,m}})
		=
		\kerdual\left(\sum_{t=1}^{m}\mKt_t\right)
	\]
After solving each SVM problem, we keep track of the value of \(\valpha_{_{\Sigma,m}}^\intercal\mKt_{_{\Sigma,m}}\valpha_{_{\Sigma,m}}\) value.
We then plot this value against the two bounds provided by \theoremref{svm-many-kernels}:
\begin{align*}
	\valpha_{_{\Sigma,m}}^\intercal\mKt_{_{\Sigma,m}}\valpha_{_{\Sigma,m}}
	&\leq
	m^{-\log_2(3)} \sum_{t=1}^{m}\valpha_t^\intercal\mKt_t\valpha_t \\
	&\leq
	m^{-\log_2(\nicefrac32)} B^2 \\
\end{align*}
\figureref{many-kernel-exper} shows the difference between the true \(\valpha_{_{\Sigma,m}}^\intercal\mKt_{_{\Sigma,m}}\valpha_{_{\Sigma,m}}\) and the two bounds above.
We can observe that the true curve decreases roughly at the same rate as our bounds.
\begin{figure}[ht]
\vskip 0.2in
\begin{center}
\centerline{\includegraphics[width=\columnwidth]{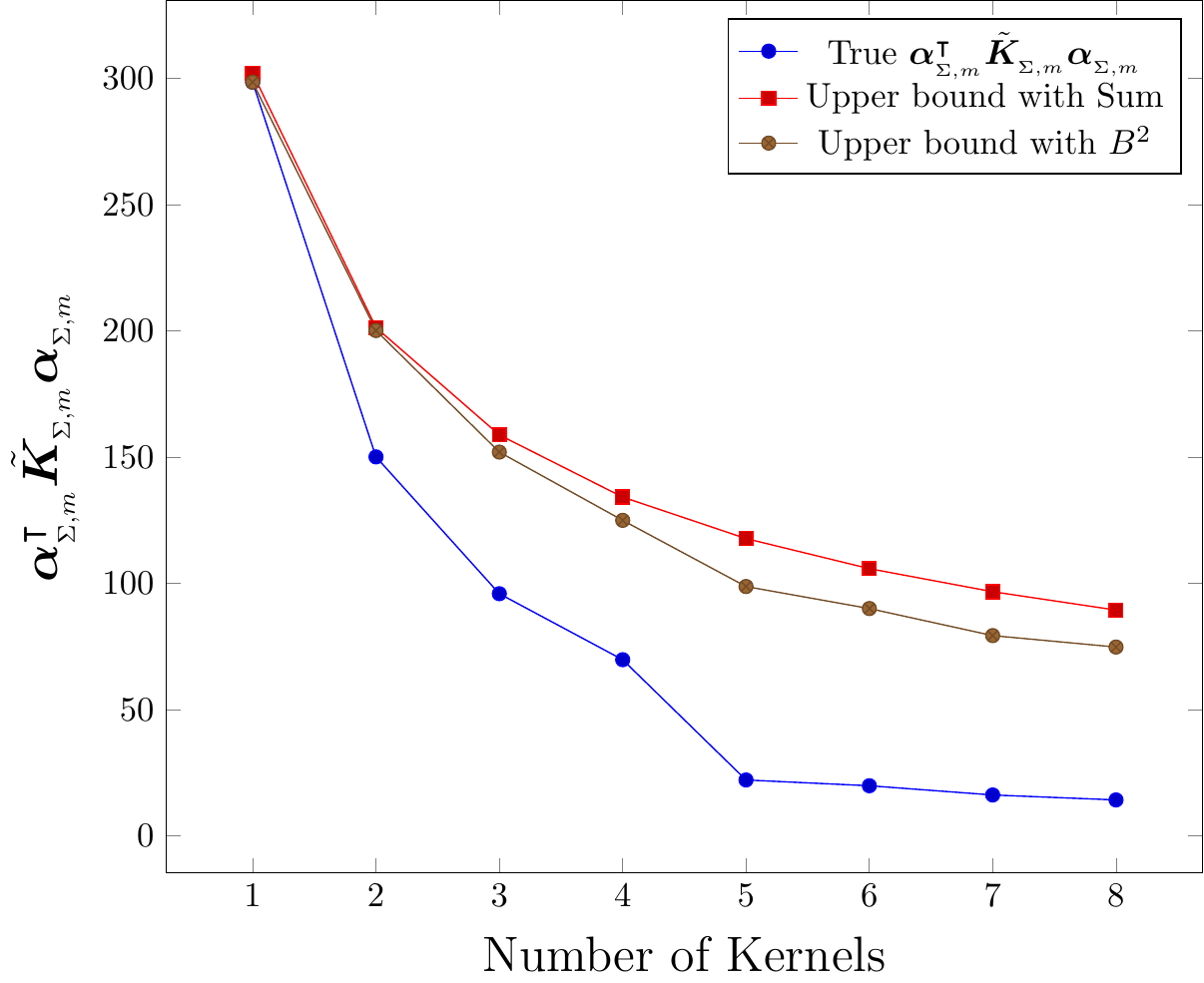}}
\caption{
Empirical value and bounds of \(\valpha_{_{\Sigma,m}}^\intercal\mKt_{_{\Sigma,m}}\valpha_{_{\Sigma,m}}\) in our experiment.
The blue curve is the empirical \(\valpha_{_{\Sigma,m}}^\intercal\mKt_{_{\Sigma,m}}\valpha_{_{\Sigma,m}}\).
The brown curve corresponds to \(m^{-\log_2(3)} \sum_{t=1}^{m}\valpha_t^\intercal\mKt_t\valpha_t\).
The red curve corresponds to \(m^{-\log_2(\nicefrac32)} B^2\).}
\label{fig:many-kernel-exper}
\end{center}
\vskip -0.2in
\end{figure}

\section{Conclusion}
Here we discuss possible directions to extend our work.
First, in the context of classification with kernel sums, we are not aware of any efficient and theoretically sound algorithms for learning which kernels to sum.
Additionally, we believe that optimality conditions such as KKT are necessary to build meaningful lower bounds in this setting.

One could also analyze the sample complexity of kernel products.
This idea is experimentally considered by \cite{duvenaud2013structure}.
This problem is notably more difficult since it requires understanding the Hadamard product of kernel matrices.

More generally, there is little existing work that leverages optimality conditions to justify assumptions made in learning problems.
In this paper, KKT tells us that we can control the quantity \(\valpha_{_\Sigma}\mKt_{_\Sigma}\valpha_{_\Sigma}\), justifying the assumptions made in prior work \cite{cortes2009new,srebro2006learning,sinha2016learning}.
We believe that this overall idea is general and applies to other convex optimization problems and classes of representer theorem problems, as well as other learning problems.

\bibliography{local}
\bibliographystyle{plain}

\newpage
\appendix

\allowdisplaybreaks

\section{Non-Separable Proof of Two Kernels (\theoremref{two-kernels-slack})}
\label{app:proof-two-kernels-slack}
In this section, we prove a theorem that mirrors that of \theoremref{two-kernels}, but with the \(\ell_2\) slack SVM.
First, we state the KKT conditions for the slack SVM.
Let \(\vr\) be the dual variables associated with the primal \(\vxi \succeq 0\) constraints.
Then, we have 8 conditions:
\begin{enumerate}
	\item \(1 - \xi_i - y_i \vw^\intercal \vphi(\vx_i) \leq 0 ~~ \forall i\in[n]\) \textit{(Primal Feasibility 1)}
	\item \(\xi_i \geq 0 ~~ \forall i\in[n]\) \textit{(Primal Feasibility 2)}
	\item \(\vw = \sum_{i=1}^{n} \alpha_i y_i \vphi(\vx_i)\) \textit{(Stationarity 1)}
	\item \(\vr = C\vxi - \valpha\) \textit{(Stationarity 2)}
	\item \(\alpha_i \geq 0 ~~ \forall i\in[n]\) \textit{(Dual Feasibility 1)}
	\item \(r_i \geq 0 ~~ \forall i\in[n]\) \textit{(Dual Feasibility 2)}
	\item \(\alpha_i (1-\xi_i-y_i\vw^\intercal\vphi(\vx_i)) = 0 ~~ \forall i\in[n]\) \textit{(Complementary Slackness 1)}
	\item \(r_i \xi_i = 0 ~~ \forall i\in[n]\) \textit{(Complementary Slackness 2)}
\end{enumerate}
We also provide two preliminary lemmas before proving the main theorem.
\begin{lemma}
\label{lemma:alpha-slack}
	Let \(\valpha,\vxi\) be the optimal solution to the \(\ell_2\) Slack Dual SVM problem with parameter \(C\).
	Then, \(\vxi=\frac1C \valpha\).
	This also implies \(\valpha^\intercal\vxi=C\normof{\vxi}_2^2\).
\end{lemma}
\begin{proof}
	First we substitute Stationarity 2 into Complementary Slackness 2:
\begin{align*}
	r_i\xi_i &= 0 \\
	(C\xi_i-\alpha_i)\xi_i &= 0 \\
	C\xi_i^2 &= \alpha_i\xi_i
\end{align*}
That is, when \(\xi_i\neq0\), we know that \(\xi_i=\frac{\alpha_i}{C}\).
This allows us to conclude that \(\xi_i\leq\frac{\alpha_i}{C}\), since both \(\alpha_i\) and \(C\) are nonnegative.
The dual problem has constraint \(\alpha_i\leq C\xi_i\), which is equivalent to \(\xi_i\geq\frac{\alpha_i}{C}\).
Hence \(\xi_i\) is both upper and lower bounded by \(\frac{\alpha_i}{C}\).
Therefore, \(\xi_i=\frac{\alpha_i}{C}\).
\end{proof}

\begin{lemma}
\label{lem:one-kernel-lemma-slack}
	Let \(\valpha,\vxi\) be the optimal solution to the \(\ell_2\) Slack Dual SVM problem on input \(\mKt\) with parameter \(C\).
	Then \(\normof{\valpha}_1 = \valpha^\intercal\mKt\valpha + C\normof{\vxi}_2^2\).
\end{lemma}
\begin{proof}
	First substitute Stationarity 1 into Complementary Slackness 1:
\begin{align*}
	0 &= \alpha_i(1-\xi_i-y_i\vw^\intercal\vphi(\vx_i)) \\
	0 &= \alpha_i\left(1-\xi_i-y_i\left(\sum\nolimits_{j=1}^{n} \alpha_j y_j \vphi(\vx_j)\right)^\intercal\vphi(\vx_i)\right) \\
	0 &= \alpha_i\left(1-\xi_i-\sum\nolimits_{j=1}^{n} \alpha_j y_i y_j \vphi(\vx_j)^\intercal\vphi(\vx_i)\right) \\
	0 &= \alpha_i\left(1-\xi_i-\sum\nolimits_{j=1}^{n} \alpha_j [\mKt]_{i,j}\right) \\
	0 &= \alpha_i-\alpha_i\xi_i-\sum_{j=1}^{n} \alpha_i\alpha_j [\mKt]_{i,j} \\
	\alpha_i &= \alpha_i\xi_i+\sum_{j=1}^{n} \alpha_i\alpha_j [\mKt]_{i,j}
\end{align*}
	Then, we sum up over all \(i\in[n]\):
\begin{align*}
	\sum_{i=1}^{n}\alpha_i &= \sum_{i=1}^{n}\alpha_i\xi_i+ \sum_{i=1}^{n}\sum_{j=1}^{n} \alpha_i\alpha_j [\mKt]_{i,j} \\
	\normof{\valpha}_1 &= \valpha^\intercal\vxi + \valpha^\intercal\mKt\valpha \\
	\normof{\valpha}_1 &= C\normof{\vxi}_2^2 + \valpha^\intercal\mKt\valpha
\end{align*}
\end{proof}
Now we prove the main theorem:
\begin{reptheorem}{two-kernels-slack}
Let \(\cS=\{(\vx_1,y_1),\ldots,(\vx_n,y_n)\}\) be a dataset.
Let \(k_1,k_2\) be kernel functions.
Define \(k_{1+2}(\cdot,\cdot)\defeq k_1(\cdot, \cdot) + k_2(\cdot, \cdot)\).
Let \(\mKt_1,\mKt_2,\mKt_{1+2}\) be their labeled kernel matrices and \(\valpha_1,\valpha_2,\valpha_{1+2}\) be the corresponding Dual SVM solutions with parameter \(C=\frac12\).
Then we have
\[
	\valpha_{1+2}^\intercal\mKt_{1+2}\valpha_{1+2}
	\leq
	\frac13(\valpha_1^\intercal\mKt_1\valpha_1 + \valpha_2^\intercal\mKt_2\valpha_2)
\]
Furthermore,
\[
	\valpha_{1+2}^\intercal\mKt_{1+2}\valpha_{1+2}
	\leq
	\frac23 \max\{\valpha_1^\intercal\mKt_1\valpha_1, \valpha_2^\intercal\mKt_2\valpha_2\}
\]
\end{reptheorem}
\begin{proof}
We start with the dual objective for \(k_{1+2}\):
\begin{align*}
	\normof{\valpha_{1+2}}_1 - \frac12 \valpha_{1+2}^\intercal\mKt_{1+2}\valpha_{1+2} - \frac12\normof{\vxi_{1+2}}_2^2
	&= \normof{\valpha_{1+2}}_1 - \frac12 \valpha_{1+2}^\intercal(\mKt_1+\mKt_2)\valpha_{1+2} - \frac12\normof{\vxi_{1+2}}_2^2 \\
	&= \left(\normof{\valpha_{1+2}}_1 - \frac12 \valpha_{1+2}^\intercal\mKt_1\valpha_{1+2} - \frac12\normof{\vxi_{1+2}}_2^2\right) \\
	&\hspace{1cm} + \left(\normof{\valpha_{1+2}}_1 - \frac12 \valpha_{1+2}^\intercal\mKt_2\valpha_{1+2} - \frac12\normof{\vxi_{1+2}}_2^2\right) \\
	&\hspace{1cm} + \left(\frac12\normof{\vxi_{1+2}}_2^2 - \normof{\valpha_{1+2}}_1\right) \\
	&\leq \left(\normof{\valpha_{1}}_1 - \frac12 \valpha_{1}^\intercal\mKt_1\valpha_{1} - \frac12\normof{\vxi_{1}}_2^2\right) \\
	&\hspace{1cm} + \left(\normof{\valpha_{2}}_1 - \frac12 \valpha_{2}^\intercal\mKt_2\valpha_{2} - \frac12\normof{\vxi_{2}}_2^2\right) \\
	&\hspace{1cm} + \left(\frac12\normof{\vxi_{1+2}}_2^2 - \normof{\valpha_{1+2}}_1\right) \\
	2\normof{\valpha_{1+2}}_1 - \frac12 \valpha_{1+2}^\intercal\mKt_{1+2}\valpha_{1+2} - \normof{\vxi_{1+2}}_2^2
	&\leq \left(\normof{\valpha_{1}}_1 - \frac12 \valpha_{1}^\intercal\mKt_1\valpha_{1} - \frac12\normof{\vxi_{1}}_2^2\right) + \left(\normof{\valpha_{2}}_1 - \frac12 \valpha_{2}^\intercal\mKt_2\valpha_{2} - \frac12\normof{\vxi_{2}}_2^2\right)
\end{align*}
By applying \lemmaref{one-kernel-lemma-slack} and some algebra, we have three useful equations:
\begin{itemize}
	\item \(2\normof{\valpha_{1+2}}_1-\frac12\valpha_{1+2}^\intercal\mKt_{1+2}\valpha_{1+2}-\normof{\vxi_{1+2}}_2^2 = \frac32 \valpha_{1+2}^\intercal\mKt_{1+2}\valpha_{1+2} + (2C-1)\normof{\vxi_{1+2}}_2^2\)
	\item \(\normof{\valpha_1}_1-\frac12\valpha_1^\intercal\mKt_1\valpha_1-\frac12\normof{\vxi_1}_2^2 = \frac12 \valpha_1^\intercal\mKt_1\valpha_1 + \frac{2C-1}{2}\normof{\vxi_1}_2^2\)
	\item \(\normof{\valpha_2}_1-\frac12\valpha_2^\intercal\mKt_2\valpha_2-\frac12\normof{\vxi_2}_2^2 = \frac12 \valpha_2^\intercal\mKt_2\valpha_2 + \frac{2C-1}{2}\normof{\vxi_2}_2^2\)
\end{itemize}
Applying these equations, we continue our inequality from before,
\begin{align*}
	\frac32 \valpha_{1+2}^\intercal\mKt_{1+2}\valpha_{1+2} + (2C-1)\normof{\vxi_{1+2}}_2^2
	&\leq \left(\frac12 \valpha_1^\intercal\mKt_1\valpha_1 + \frac{2C-1}{2}\normof{\vxi_1}_2^2\right) + \left(\frac12 \valpha_2^\intercal\mKt_2\valpha_2 + \frac{2C-1}{2}\normof{\vxi_2}_2^2\right) \\
	\frac32 \valpha_{1+2}^\intercal\mKt_{1+2}\valpha_{1+2}
	&\leq \frac12 \left(\valpha_1^\intercal\mKt_1\valpha_1 + \valpha_2^\intercal\mKt_2\valpha_2\right) + \frac{2C-1}{2}\left(\normof{\vxi_1}_2^2 + \normof{\vxi_2}_2^2 - 2\normof{\vxi_{1+2}}_2^2\right) \\
	\frac32 \valpha_{1+2}^\intercal\mKt_{1+2}\valpha_{1+2} &= \frac12 \left(\valpha_1^\intercal\mKt_1\valpha_1 + \valpha_2^\intercal\mKt_2\valpha_2\right) + 0 \\
	\valpha_{1+2}^\intercal\mKt_{1+2}\valpha_{1+2} &= \frac13 \left(\valpha_1^\intercal\mKt_1\valpha_1 + \valpha_2^\intercal\mKt_2\valpha_2\right)
\end{align*}
In the second to last line, we recall that \(C=\frac12\), which implies \(2C-1=0\).
\end{proof}

\section{Proof of Many Kernels (\theoremref{svm-many-kernels})}
\label{app:svm-many-kernels}

\begin{reptheorem}{svm-many-kernels}
Let \(\cS=\{(\vx_1,y_1),\ldots,(\vx_n,y_n)\}\) be a dataset.
Let \(k_1,k_2,\ldots,k_m\) be kernel functions.
Define \(k_{_\Sigma}(\cdot,\cdot)\defeq \sum_{t=1}^{m} k_t(\cdot, \cdot)\).
Let \(\mKt_1,\ldots,\mKt_m,\mKt_{_\Sigma}\) be their labeled kernel matrices and \(\valpha_1,\ldots,\valpha_m,\valpha_{_\Sigma}\) be the corresponding Dual SVM solutions.
Then we have
\[
	\valpha_{_\Sigma}^\intercal\mKt_{_\Sigma}\valpha_{_\Sigma}
	\leq
	3 m^{-\log_2(3)} \sum_{t=1}^{m}\valpha_t^\intercal\mKt_t\valpha_t
\]
Furthermore
\[
	\valpha_{_\Sigma}^\intercal\mKt_{_\Sigma}\valpha_{_\Sigma}
	\leq
	3 m^{-\log_2(\nicefrac32)} \max_{t\in[m]} \valpha_t^\intercal\mKt_t\valpha_t
\]
\end{reptheorem}
\begin{proof}
Let \(\ell\defeq\ceil{\log_2(m)}\) be the length of labels we give our base kernels.
Now, rename each kernel \(k_t\) with the length \(\ell\) bitstring representation of the number \(t\).
For instance, if \(\ell=4\) then we rename \(k_6\) to \(k_{0110}\).
For every length \(\ell-1\) bitstring \(b_0b_1 \ldots b_{\ell-1}\), define a new kernel
\[k_{b_0b_1 \ldots b_{\ell-1}}(\cdot,\cdot)\defeq k_{b_0b_1\ldots b_{\ell-1}0}(\cdot,\cdot) + k_{b_0b_1\ldots b_{\ell-1}1}(\cdot,\cdot)\]
Repeat this process of labeling with length \(\ell - 2\) bitstrings and so on until we have defined \(k_0\) and \(k_1\).
Lastly, we define \[k_{_\Sigma}(\cdot,\cdot) = k_0(\cdot,\cdot) + k_1(\cdot,\cdot)=\sum_{t=1}^{m} k_t(\cdot,\cdot)\].

Now, recall \theoremref{two-kernels} (or \theoremref{two-kernels-slack} if we are using the SVM with slack).
Let \([b_\ell]\defeq\{b_0 \ldots b_\ell | b\in\{0,1\}\}\) denote the set of all length \(\ell\) bitstrings.
Also, for every kernel \(k_{b_0 \ldots b_j}\), compute the associated kernel matrix \(\mKt_{b_0 \ldots b_j}\) and dual solution vector \(\valpha_{b_0 \ldots b_j}\).
\begin{claim}
\label{clm:kernel-sum-induction}
Fix \(j\in[\ell-1]\).
Then \[\valpha_{_\Sigma}^\intercal\mKt_{_\Sigma}\valpha_{_\Sigma} \leq \left(\frac13\right)^j \sum_{b_0 \ldots b_j \in [b_j]} \valpha_{b_0 \ldots b_j}^\intercal \mKt_{b_0 \ldots b_j} \valpha_{b_0 \ldots b_j}\]
\end{claim}
This claim follows from induction.
In the base case, \(j=1\), and \theoremref{two-kernels} tells us that \(\valpha_{_\Sigma}^\intercal\mKt_{_\Sigma}\valpha_{_\Sigma} \leq \frac13 (\valpha_{0}^\intercal\mKt_{0}\valpha_{0} + \valpha_{1}^\intercal\mKt_{1}\valpha_{1})\), matching the claim.
Now, assume the claim holds for \(j-1\).
Then,
\begin{align*}
	\valpha_{_\Sigma}^\intercal\mKt_{_\Sigma}\valpha_{_\Sigma} 
	&\leq \left(\frac13\right)^j \sum_{b_0 \ldots b_j \in [b_j]} \valpha_{b_0 \ldots b_j}^\intercal \mKt_{b_0 \ldots b_j} \valpha_{b_0 \ldots b_j} \\
	&\leq \left(\frac13\right)^j \sum_{b_0 \ldots b_j \in [b_j]} \frac13 (\valpha_{b_0 \ldots b_j 0}^\intercal \mKt_{b_0 \ldots b_j 0} \valpha_{b_0 \ldots b_j 0} + \valpha_{b_0 \ldots b_j 1}^\intercal \mKt_{b_0 \ldots b_j 1} \valpha_{b_0 \ldots b_j 1}) \\
	&= \left(\frac13\right)^{j+1} \sum_{b_0 \ldots b_{j+1} \in [b_{j+1}]} \valpha_{b_0 \ldots b_{j+1}}^\intercal \mKt_{b_0 \ldots b_{j+1}} \valpha_{b_0 \ldots b_{j+1}}
\end{align*}
This completes the proof of the claim.

Now we need to be careful when moving to the length \(\ell\) kernel labels because if \(m\) is not a power of two, then only some of the kernels have a length \(\ell\) label.
Let \(\cA\) be the set of all base kernels that have a length \(\ell-1\) label.
Let \(\cB\) be the rest of the base kernels, with a length \(\ell\) label.
By \claimref{kernel-sum-induction}, we know that
\begin{align*}
	\valpha_{_\Sigma}^\intercal\mKt_{_\Sigma}\valpha_{_\Sigma}
	&\leq \left(\frac13\right)^{\ell-1} \sum_{b_0 \ldots b_{\ell-1} \in [b_{\ell-1}]} \valpha_{b_0 \ldots b_{\ell-1}}^\intercal \mKt_{b_0 \ldots b_{\ell-1}} \valpha_{b_0 \ldots b_{\ell-1}} \\
	&= \sum_{b_0 \ldots b_{\ell-1} \in [b_{\ell-1}]} \left(\frac13\right)^{\ell-1} \valpha_{b_0 \ldots b_{\ell-1}}^\intercal \mKt_{b_0 \ldots b_{\ell-1}} \valpha_{b_0 \ldots b_{\ell-1}} \\
	&= \sum_{\substack{b_0 \ldots b_{\ell-1} \in [b_{\ell-1}] : \\ k_{b_0 \ldots b_{\ell-1}}\in\cA }} \left(\frac13\right)^{\ell-1} \valpha_{b_0 \ldots b_{\ell-1}}^\intercal \mKt_{b_0 \ldots b_{\ell-1}} \valpha_{b_0 \ldots b_{\ell-1}}
	+ \sum_{\substack{b_0 \ldots b_{\ell-1} \in [b_{\ell-1}] : \\ k_{b_0 \ldots b_{\ell-1}}\notin\cA }} \left(\frac13\right)^{\ell-1} \valpha_{b_0 \ldots b_{\ell-1}}^\intercal \mKt_{b_0 \ldots b_{\ell-1}} \valpha_{b_0 \ldots b_{\ell-1}}\\
	&\leq \sum_{\substack{b_0 \ldots b_{\ell-1} \in [b_{\ell-1}] : \\ k_{b_0 \ldots b_{\ell-1}}\in\cA }} \left(\frac13\right)^{\ell-1} \valpha_{b_0 \ldots b_{\ell-1}}^\intercal \mKt_{b_0 \ldots b_{\ell-1}} \valpha_{b_0 \ldots b_{\ell-1}}
	+ \sum_{\substack{b_0 \ldots b_{\ell} \in [b_{\ell}] : \\ k_{b_0 \ldots b_{\ell}}\in\cB }} \left(\frac13\right)^{\ell} \valpha_{b_0 \ldots b_{\ell}}^\intercal \mKt_{b_0 \ldots b_{\ell}} \valpha_{b_0 \ldots b_{\ell}}\\
	&\leq \sum_{\substack{b_0 \ldots b_{\ell-1} \in [b_{\ell-1}] : \\ k_{b_0 \ldots b_{\ell-1}}\in\cA }} \left(\frac13\right)^{\ell-1} \valpha_{b_0 \ldots b_{\ell-1}}^\intercal \mKt_{b_0 \ldots b_{\ell-1}} \valpha_{b_0 \ldots b_{\ell-1}}
	+ \sum_{\substack{b_0 \ldots b_{\ell} \in [b_{\ell}] : \\ k_{b_0 \ldots b_{\ell}}\in\cB }} \left(\frac13\right)^{\ell-1} \valpha_{b_0 \ldots b_{\ell}}^\intercal \mKt_{b_0 \ldots b_{\ell}} \valpha_{b_0 \ldots b_{\ell}}\\
	&\leq \left(\frac13\right)^{\ell-1} \left(\sum_{\substack{b_0 \ldots b_{\ell-1} \in [b_{\ell-1}] : \\ k_{b_0 \ldots b_{\ell-1}}\in\cA }} \valpha_{b_0 \ldots b_{\ell-1}}^\intercal \mKt_{b_0 \ldots b_{\ell-1}} \valpha_{b_0 \ldots b_{\ell-1}}
	+ \sum_{\substack{b_0 \ldots b_{\ell} \in [b_{\ell}] : \\ k_{b_0 \ldots b_{\ell}}\in\cB }} \valpha_{b_0 \ldots b_{\ell}}^\intercal \mKt_{b_0 \ldots b_{\ell}} \valpha_{b_0 \ldots b_{\ell}}\right)\\
	&= \left(\frac{1}{3}\right)^{\ell-1} \sum_{t=1}^m \valpha_t^\intercal\mKt_t\valpha_t
\end{align*}
Where the second inequalty applies \theoremref{two-kernels} and the last equality uses the fact that all base kernels are in either \(\cB\) or \(\cA\).

Lastly, recall that \(\ell=\ceil{\log_2(m)}\).
\[
	\left(\frac13\right)^{\ell-1}
	=
	3^{1-\ceil{\log_2(m)}}
	=
	3
	\cdot
	3^{-\ceil{\log_2(m)}}
	\leq 3 \cdot 3^{-\log_2(m)}
	= 3 \cdot m^{-\log_2(3)}
\]
Therefore, overall, we have
\[
	\valpha_{_\Sigma}^\intercal\mKt_{_\Sigma}\valpha_{_\Sigma}
	\leq
	3 m^{-\log_2(3)}
	\sum_{t=1}^m \valpha_t^\intercal\mKt_t\valpha_t
\]
\end{proof}

\section{Proof of Kernel Sum Rademacher (\theoremref{rademacher-kernel-sum})}
\label{app:rademacher-kernel-sum}

\begin{reptheorem}{rademacher-kernel-sum}
Let \(\cS=\{(\vx_1,y_1),\ldots,(\vx_n,y_n)\}\) be a dataset.
Let \(k_1,\ldots,k_m\) be kernel functions.
Define \(k_{_\Sigma}(\cdot,\cdot)\defeq\sum_{t=1}^{m}k_t(\cdot,\cdot)\).
Let \(\mKt_1,\ldots,\mKt_m,\mKt_{_\Sigma}\) be their labeled kernel matrices and \(\valpha_1,\ldots,\valpha_m,\valpha_{_\Sigma}\) be the corresponding Dual SVM solutions.
Then,
\[
	\hat\fR_\cS(\cF_{_\Sigma})
	\leq
	\frac1n\sqrt{3m^{-\log_2(3)} \left(\sum_{t=1}^{m}\tr[\mKt_t]\right) \sum_{t=1}^{m}\valpha_t^\intercal\mKt_t\valpha_t}
\]
Further, if we assume \(\valpha_t^\intercal\mKt_t\valpha_t\leq B^2\) and \(k_t(\vx_i,\vx_i) \leq R^2\) for all \(t\in[m], i\in[n]\), then
\[
	\hat\fR_\cS(\cF_{_\Sigma})
	\leq
	\frac{BR}{\sqrt n} ~ \sqrt{3m^{(1-\log_2(\nicefrac32))}}
\]
\end{reptheorem}
This proof very closely parallels that of Lemma 22 in \cite{bartlett2002rademacher}.
We produce the entire proof here for completeness.
First, note that
\[
	\cF_{_\Sigma} \subseteq \{\vx\mapsto\vw_{_\Sigma}^\intercal\vphi_{_\Sigma} | \normof{\vw_{_\Sigma}}_2 \leq B_{_\Sigma}\}
\]
Where \(\vphi_{_\Sigma}\) is the concatenation of the feature spaces associated with each of the \(m\) kernels, and \(B_{_\Sigma}^2 = \valpha_{_\Sigma}^\intercal\mKt_{_\Sigma}\valpha_{_\Sigma}\).
Then,
\begin{align*}
	\hat\fR_\cS(\cF_{_\Sigma})
	&\leq \frac1n \E_{\vsigma} \left[ \sup_{\normof{\vw_{_\Sigma}} \leq B_{_\Sigma}} \left(\vw_{_\Sigma}^\intercal \sum_{i=1}^{n} \sigma_iy_i\vphi_{_\Sigma}(\vx_i) \right) \right] \\
	&= \frac{B_{_\Sigma}}{n} \E_{\vsigma} \left[ \normoflr{\sum_{i=1}^{n} \sigma_iy_i\vphi_{_\Sigma}(\vx_i)}_2 \right] \\
	&= \frac{B_{_\Sigma}}{n} \E_{\vsigma} \left[ \sqrt{\normoflr{\sum_{i=1}^{n} \sigma_iy_i\vphi_{_\Sigma}(\vx_i)}_2^2} ~ \right] \\
	&= \frac{B_{_\Sigma}}{n} \E_{\vsigma} \left[ \sqrt{\sum\nolimits_{i,j=1}^{n} \sigma_i\sigma_j[\mKt_{_\Sigma}]_{i,j} } ~ \right] \\
	&\leq \frac{B_{_\Sigma}}{n} \sqrt{ \E_{\vsigma} \left[\sum\nolimits_{i,j=1}^{n} \sigma_i\sigma_j[\mKt_{_\Sigma}]_{i,j} \right]  } \\
	&= \frac{B_{_\Sigma}}{n} \sqrt{ \E_{\vsigma} \left[\sum\nolimits_{i=1}^{n} \sigma_i^2[\mKt_{_\Sigma}]_{i,i} \right]  } \\
	&= \frac{B_{_\Sigma}}{n} \sqrt{ \E_{\vsigma} \left[\sum\nolimits_{i=1}^{n} [\mKt_{_\Sigma}]_{i,i} \right]  } \\
	&= \frac{B_{_\Sigma}}{n} \sqrt{ \tr[\mKt_{_\Sigma}]  } \\
	&= \frac{B_{_\Sigma}}{n} \sqrt{ \sum_{t=1}^m \tr[\mKt_t]  } \\
	&\leq \frac{1}{n} ~ \cdot ~ \sqrt{3m^{(1-\log_2(3))} \sum_{t=1}^{m} \valpha_t^\intercal\mKt_t\valpha_t} ~ \cdot ~ \sqrt{\sum_{t=1}^m \tr[\mKt_t]  } \\
	&= \frac1n\sqrt{3m^{-\log_2(3)} \left(\sum_{t=1}^{m}\tr[\mKt_t]\right) \sum_{t=1}^{m}\valpha_t^\intercal\mKt_t\valpha_t} \\
\end{align*}
The second inequality is Jensen's, and the last inequality is \theoremref{svm-many-kernels}.
This completes the first part of the proof.
We can then substitute in \(B^2\) and \(R^2\):
\begin{align*}
	\hat\fR_\cS(\cF_{_\Sigma})
	&\leq \frac1n\sqrt{3m^{-\log_2(3)} \left(\sum_{t=1}^{m}\tr[\mKt_t]\right) \sum_{t=1}^{m}\valpha_t^\intercal\mKt_t\valpha_t} \\
	&\leq \frac1n\sqrt{3m^{-\log_2(3)} \left(\sum_{t=1}^{m} nR^2\right) \sum_{t=1}^{m}B^2} \\
	&= \frac1n\sqrt{3m^{-\log_2(3)} \cdot mnR^2 \cdot mB^2} \\
	&= \frac{BR}{\sqrt n} ~ \sqrt{3m^{(1-\log_2(\nicefrac32))}}
\end{align*}

\section{Proof of Learning Kernels (\theoremref{learn-kernel-bound})}
\label{app:rademacher-kernel-learn}

\begin{reptheorem}{learn-kernel-bound}
Let \(\cS=\{(\vx_1,y_1),\ldots,(\vx_n,y_n)\}\) be a dataset.
Let \(k_1,\ldots,k_m\) be kernel functions.
Consider any \(\cP\subseteq[m]\).
Define \(k_{_\cP}(\cdot,\cdot)\defeq\sum_{t\in\cP}k_t(\cdot,\cdot)\).
Let \(\mKt_1,\ldots,\mKt_m,\mKt_{_\cP}\) be their labeled kernel matrices and \(\valpha_1,\ldots,\valpha_m,\valpha_{_\cP}\) be the corresponding Dual SVM solutions.
Assume \(k_t(\vx_i,\vx_i) \leq R^2\) and \(\valpha_t^\intercal\mKt_t\valpha_t\leq B^2\) for all \(t\in[m]\) and \(i\in[n]\).
Then,
\[
	\hat\fR_\cS(\cF_{_\cP})
	\leq
	\frac{BR \sqrt{3e\eta_0 ~ m^{(1-\log_2(\nicefrac{3}{2}))} \ceil{\ln(m)}}}{\sqrt n}
\]
where \(\eta_0=\frac{23}{22}\).
\end{reptheorem}
This proof closely follows that of Theorem 1 in \cite{cortes2009new}.
\begin{proof}
Let \(s\defeq \abs{\cP}\).
Let \(\vw_{_\cP}\) be the optimal Primal SVM solution using subset of kernels \cP.
Note that \(\vw_{_\cP}\) is a \textit{concatenation} of \(s\) labeled and scaled feature vectors.
To be precise, let \(\vphi_t\) be the feature map for the \(t^{th}\) kernel and define \(\vw_t\defeq\sum_{i=1}^{n}\alpha_iy_i\vphi_t(\vx_i)\).
Then \(\vw_{_\cP}=\sbmat{\vw_{t_1}^\intercal & \ldots & \vw_{t_{s}}^\intercal}^\intercal\), where \(t_i\) is the \(i^{th}\) smallest element of \cP.

Consider some \(q,r>1\) such that \(\frac1q + \frac1r = 1\).
Then,
\begin{align*}
	\hat\fR_\cS(\cF_{\cP})
	&\defeq \frac1n \E_{\vsigma} \left[ \sup_{h\in\cF_{_\Sigma}} \sum_{i=1}^{n} \sigma_i h(\vx_i,y_i) \right] \\
	&\leq \frac1n \E_{\vsigma} \left[\sup_{s\in[m]} \sup_{\abs{\cP}=s} \sup_{\vw_{_\cP}} \vw_{_\cP}^\intercal \left(\sum\nolimits_{i=1}^{n} \sigma_i y_i \vphi_{_\cP}(\vx_i)\right) \right] \\
	&\leq \frac1n \E_{\vsigma}
		\left[
			\sup_{s\in[m]}
			\sup_{\abs{\cP}=s}
			\sup_{\vw_{_\cP}}
			\left(\sum\nolimits_{t\in\cP} \normof{\vw_t}_2^q \right)^{\nicefrac{1}{q}}
			\left(\sum\nolimits_{t\in\cP} \normoflr{\sum\nolimits_{i=1}^{n}\sigma_i y_i \vphi_t(\vx_i)}_2^r\right)^{\nicefrac{1}{r}}
		\right] \\
	&\leq \frac1n \E_{\vsigma}
		\left[
			\sup_{s\in[m]}
			\sup_{\abs{\cP}=s}
			\sup_{\vw_{_\cP}}
			\left(\sum\nolimits_{t\in\cP} \normof{\vw_t}_2^q \right)^{\nicefrac{1}{q}}
			\left(\sum\nolimits_{t=1}^m \normoflr{\sum\nolimits_{i=1}^{n}\sigma_i y_i \vphi_t(\vx_i)}_2^r\right)^{\nicefrac{1}{r}}
		\right] \\
	&=  \frac1n
		\left[
			\sup_{s\in[m]}
			\sup_{\abs{\cP}=s}
			\sup_{\vw_{_\cP}}
			\left(\sum\nolimits_{t=1}^{m} \normof{\vw_t}_2^q \right)^{\nicefrac{1}{q}}
		\right]
		\cdot
		\E_{\vsigma}
		\left[
			\left(\sum\nolimits_{t=1}^{m} \normoflr{\sum\nolimits_{i=1}^{n}\sigma_i y_i \vphi_t(\vx_i)}_2^r\right)^{\nicefrac{1}{r}}
		\right]
\end{align*}
The third line follows exactly from Lemma 5 in \cite{cortes2009new}.
We bound both terms separately.
We only substantially differ from the original proof in bounding the first term.
To start, note that \(f(x)=x^{\nicefrac1q}\) is subadditive for \(\nicefrac1q<1\):
\begin{align*}
	\left(\sum\nolimits_{t\in\cP} \normof{\vw_t}_2^q \right)^{\nicefrac{1}{q}}
	&\leq \sum_{t\in\cP} \left(\normof{\vw_t}_2^q\right)^{\nicefrac1q} \\
	&= \sum_{t\in\cP} \normoflr{\sum_{i=1}^{n}\alpha_iy_i\vphi_t(\vx_i)}_2 \\
	&= s \sum_{t\in\cP} \frac1s \sqrt{\normoflr{\sum_{i=1}^{n}\alpha_iy_i\vphi_t(\vx_i)}_2^2} \\
	&\leq s \sqrt{\sum_{t\in\cP} \frac1s \normoflr{\sum_{i=1}^{n}\alpha_iy_i\vphi_t(\vx_i)}_2^2} \\
	&= \sqrt{s \cdot \sum_{t\in\cP} \valpha_{_\cP}^\intercal\mKt_t\valpha_{_\cP}} \\
	&= \sqrt{s \cdot \valpha_{_\cP}^\intercal\mKt_{_\cP}\valpha_{_\cP}} \\
	&\leq \sqrt{s \cdot 3s^{^{-\log_2(\nicefrac32)}} B^2} \\
	&= B\sqrt{3s^{^{(1-\log_2(\nicefrac32))}}}
\end{align*}
The second inequality follows from Jensen's, and the last inequality is \theoremref{two-kernels}.

We start our bound of the second term by applying Jensen's Inequality:
\begin{align*}
	\E_{\vsigma}
	\left[
		\left(\sum\nolimits_{t=1}^{m} \normoflr{\sum\nolimits_{i=1}^{n}\sigma_i \vphi_t(\vx_i)}_2^r\right)^{\nicefrac{1}{r}}
	\right]
	&\leq
	\left(
	\E_{\vsigma}
	\left[
		\sum\nolimits_{t=1}^{m} \normoflr{\sum\nolimits_{i=1}^{n}\sigma_i \vphi_t(\vx_i)}_2^r
	\right]
	\right)^{\nicefrac{1}{r}} \\
	&=
	\left(
	\sum\nolimits_{t=1}^{m}
	\E_{\vsigma} \left[ \normoflr{\sum\nolimits_{i=1}^{n}\sigma_i \vphi_t(\vx_i)}_2^r \right]
	\right)^{\nicefrac{1}{r}}
\end{align*}
We detour to bound the inner expectation.
Assume that \(r\) is an even integer.
That is, \(r=2p\) for some integer \(p\).
\begin{align*}
	\E_{\vsigma} \left[ \normoflr{\sum\nolimits_{i=1}^{n}\sigma_i \vphi_t(\vx_i)}_2^r \right]
	&= \E_{\vsigma} \left[ \left( \sum\nolimits_{i,j=1}^{n} \sigma_i\sigma_j k_t(\vx_i,\vx_j) \right)^p \right] \\
	&= \E_{\vsigma} \left[ \left( \vsigma^\intercal\mKt_t\vsigma \right)^p\right] \\
	&\leq \left(\eta_0p\tr[\mKt]\right)^p
\end{align*}
Where the last line follows from Lemma 1 in \cite{cortes2010generalization}, where \(\eta_0=\frac{23}{22}\).
Returning to the bound of the second term,
\begin{align*}
	\E_{\vsigma}
	\left[ \left(\sum\nolimits_{t=1}^{m} \normoflr{\sum\nolimits_{i=1}^{n}\sigma_i \vphi_t(\vx_i)}_2^r\right)^{\nicefrac{1}{r}} \right]
	&\leq
	\left( \sum\nolimits_{t=1}^{m} \E_{\vsigma} \left[ \normoflr{\sum\nolimits_{i=1}^{n}\sigma_i \vphi_t(\vx_i)}_2^r \right] \right)^{\nicefrac{1}{2p}}	\\
	&\leq \left( \sum\nolimits_{t=1}^{m} \left(\eta_0p\tr[\mKt_t]\right)^p \right)^{\nicefrac{1}{2p}} \\
	&\leq \left( \sum\nolimits_{t=1}^{m} \left(\eta_0p ~ nR^2 \right)^p \right)^{\nicefrac{1}{2p}} \\
	&= \left( m \left(\eta_0p ~ nR^2 \right)^p \right)^{\nicefrac{1}{2p}} \\
	&= m^{\nicefrac1{2p}} \sqrt{\eta_0p ~ nR^2}
\end{align*}
By differentiating, we find that \(p=\ln(m)\) minimizes this expression.
We required \(p\) to be an integer, so we instead take \(p=\ceil{\ln(m)}\).
\begin{align*}
	\E_{\vsigma}
	\left[
		\left(\sum\nolimits_{t=1}^{m} \normoflr{\sum\nolimits_{i=1}^{n}\sigma_i \vphi_t(\vx_i)}_2^r\right)^{\nicefrac{1}{r}}
	\right]
	&\leq R m^{\nicefrac1{2p}} \sqrt{\eta_0pn} \\
	&= R m^{\frac{1}{2\ceil{\ln(m)}}} \sqrt{\eta_0 \ceil{\ln(m)}n} \\
	&\leq R \sqrt{e\eta_0 \ceil{\ln(m)}n}
\end{align*}
Combining the first and second terms' bounds, we return to the bound of the Rademacher complexity itself:
\begin{align*}
	\hat\fR_\cS(\cF_{_\cP})
	&\leq
		\frac1n
		\left[
			\sup_{s\in[m]}
			\sup_{\abs{\cP}=s}
			\sup_{\vw_{_\cP}}
			\left(\sum\nolimits_{t=1}^{m} \normof{\vw_t}_2^q \right)^{\nicefrac{1}{q}}
		\right]
		\cdot
		\E_{\vsigma}
		\left[
			\left(\sum\nolimits_{t=1}^{m} \normoflr{\sum\nolimits_{i=1}^{n}\sigma_i y_i \vphi_t(\vx_i)}_2^r\right)^{\nicefrac{1}{r}}
		\right] \\
	&\leq
		\frac1n
		\left[ \sup_{s\in[m]} B \sqrt{3s^{^{(1-\log_2(\nicefrac{3}{2}))}}}\right]
		\cdot
		\left[R \sqrt{e\eta_0 \ceil{\ln(m)}n}\right] \\
	&=
		\frac1n
		\left[ B \sqrt{3m^{^{(1-\log_2(\nicefrac{3}{2}))}}}\right]
		\cdot
		\left[R \sqrt{e\eta_0 \ceil{\ln(m)}n}\right] \\
	&=
		\frac{BR \sqrt{3e\eta_0 ~ m^{^{(1-\log_2(\nicefrac32))}} \ceil{\ln(m)}}}{\sqrt n}
\end{align*}
\end{proof}

\end{document}